%% file: arxiv_submission.tex
\documentclass{article}
\usepackage{amsmath}
\usepackage[amsthm]{ntheorem}
\usepackage{stmaryrd}
\usepackage{caption}

\usepackage{amssymb}
\usepackage{hyperref}
\usepackage{graphicx}
\usepackage{subcaption}
\usepackage{float}
\usepackage[round]{natbib}
\usepackage{mathabx}

\usepackage{thmtools,thm-restate}

\usepackage{geometry}
\newgeometry{
    textheight=9in,
    textwidth=5.5in,
    top=1in,
    headheight=12pt,
    headsep=25pt,
    footskip=30pt
  }
\date{}

\usepackage{siunitx}
\DeclareMathOperator{\shrink}{ST}

\newcommand{\rr}{J}

\newcommand{\x}{x}
\newcommand{\y}{y}
\newcommand{\X}{{\mathbf X}}
\newcommand{\Y}{{\mathbf y}}
\newcommand{\w}{w}

\newcommand{\ws}{w^\star}
\newcommand{\thetas}{\theta^\star}
\newcommand{\zs}{z^\star}

\makeatletter
\newcommand\nocaption{%
    \renewcommand\p@subfigure{}
    \renewcommand\thesubfigure{\thefigure\alph{subfigure}}
}
\makeatother

\usepackage{xcolor}
\definecolor{linkcolor}{RGB}{83,83,182}
\definecolor{citecolor}{RGB}{128,0,128}
\hypersetup{
    colorlinks=true,
    citecolor=citecolor,
    linkcolor=linkcolor,
    urlcolor=linkcolor}

\usepackage{MacroBook}
\graphicspath{{images/}}

\title{Iterative regularization for convex regularizers\\
}

\author{
\normalsize
Cesare Molinari${}^1$ \quad \quad Mathurin Massias${}^2$ \quad \quad Lorenzo Rosasco${}^{124}$ \quad \quad Silvia Villa${}^3$ \\[2mm]
\normalsize ${}^1$Istituto Italiano di Tecnologia  \\
\normalsize${}^2$MaLGa, DIBRIS, Universit\`a degli Studi di Genova \\
\normalsize${}^3$MaLGa, DIMA, Universit\`a degli Studi di Genova \\
\normalsize ${}^4$Center for Brains, Minds and Machines, Massachussets Institute of Technology
}
\usepackage{setspace}

\begin{document}
\maketitle

\begin{abstract}
  We study iterative regularization for linear models, when the bias is convex but not necessarily strongly convex.
  We characterize the stability properties of a primal-dual gradient based approach, analyzing its convergence in the presence of worst case deterministic noise.
  As a main example, we specialize and illustrate the results for the problem of robust sparse recovery.
  Key to our analysis is a combination of ideas from regularization theory and optimization in the presence of errors.
  Theoretical results are complemented by experiments showing that state-of-the-art performances can be achieved with considerable computational speed-ups.
\end{abstract}

\section{Introduction}
Machine learning often reduces to estimating some model parameters.  This approach raises at least two orders of questions:
first, multiple solutions may exist, amongst which a specific one must be selected;
second, potential instabilities with respect to noise and sampling must be controlled.\\
A classical way to achieve both goals is to consider explicitly penalized or constrained objective functions.
In machine learning, this leads to regularized empirical risk minimization \citep{Shalev-Shwartz_Ben-David14}.
A more recent approach is based on directly exploiting an iterative optimization procedure for an unconstrained/unpenalized problem.

This approach is shared by several related ideas.
One is implicit regularization \citep{mahoney2012approximate,gunasekar2017implicit}, stemming from the observation that the bias is controlled increasing the number of iterations, just like in penalized methods it is controlled decreasing the penalty parameter.
Another one is early stopping \citep{yao2007early,raskutti}, putting emphasis on the fact that running the iterates to convergence might lead to instabilities in the presence of noise.
Yet another, and more classical, idea is iterative regularization, where both aspects (convergence and stability) are considered to be relevant \citep{engl1996regularization,kaltenbacher2008iterative}.
This approach naturally blends modeling and numerical aspects, often improving computational efficiency, while retaining good prediction accuracy \citep{yao2007early}.
Another reason of interest is that iterative regularization may be one of the mechanisms explaining generalization in deep learning \citep{neyshabur2017geometry,gunasekar2017implicit,arora2019implicit,vavskevivcius2020statistical}.

A classic illustrative example is gradient descent for linear least squares.
The latter, if suitably initialized, converges (is biased) to the minimal Euclidean norm solution.
Moreover, its stability is controlled along the iterative process, allowing to derive early stopping criterions depending on the noise \citep{engl1996regularization,raskutti}.
There are a number of developments of these basic results.
For example, one line of work has considered extensions to other gradient-based methods, such as stochastic and accelerated gradient descent \citep{zhang2005boosting,moulines2011non,rosasco2015learning,PagRos19}.
Another line of work has considered classification problems \citep{gunasekar2017implicit,soudry2018implicit} and also
nonlinear models, such as deep networks \citep{neyshabur2017geometry}, see also \cite{kaltenbacher2008iterative} for results for non linear inverse problems.

In this work, we are interested in iterative regularization procedures where the considered bias is not the Euclidean norm but rather a general convex functional.
The question is to determine whether or not there exists an iteration analogous to gradient descent for such general bias.
This question has been studied when the bias is \emph{strongly} convex.
In this case, linearized Bregman iterations (a.k.a. mirror descent) can be used \citep{burger2007error,gunasekar2018characterizing}.
For this approach accelerated algorithms \citep{matet2017don} have also been considered and studied.
Finally, approaches have also been studied based on  diagonal methods \citep{garrigos2018iterative} and their acceleration \citep{calatroni2019accelerated}.
The general convex case, even for linear models, is much less understood.
There have been studies for ADMM/Bregman iteration, but the procedure requires solving a nontrivial optimization problem at each iteration \citep{burger2007error}.
Further, stability and convergence results are proved only in terms of Bregman divergence, which in general is a weak result.
Interestingly, various recent results study  iterative regularization for sparse recovery, where the bias is defined by an $\ell_1$ norm \citep{agarwal2012stochastic,osher2016sparse,Vaskevicius_Kanade_Rebeschini19}.

In this paper we propose and study an efficient algorithm for general convex bias beyond $\ell_1$ norm.
Indeed, we adapt the Chambolle and Pock (CP) algorithm, popular in imaging \citep{Chambolle_Pock11}, and study its {iterative} regularization properties. The CP algorithm  is a first order primal-dual method, thus easy to implement and requiring only matrix-vector multiplications and proximity operators.
In the setting of  linear models with worst case errors, our analysis provides dimension free convergence and stability results  in terms of  both Bregman divergence and approximate feasibility. A combination of these two results allows to derive strong convergence results in the $\ell_1$ norm case.
The proof relies on results from the analysis of primal-dual methods with errors  \citep{rasch2020inexact}.
From our general results, several special cases can be derived and we discuss as an example sparse recovery, proving dimension free estimates in norm.
In the experimental section, we investigate the proposed method and show state-of-the-art performances with significant computation savings compared to the Tikhonov approach.

\noindent{\bf Notation,}
The set of integers from 1 to $n$ is $[n]$.
Let $f\colon\mathbb{R}^n\to \mathbb{R}\cup\{+\infty\}$ and $J\colon\mathbb{R}^p\to \mathbb{R}\cup\{+\infty\}$ be proper, convex, and lower semicontinuous.
The subdifferential of $J$ at $\w\in\mathbb{R}^p$ is $\partial J(\w)$.
The Bregman divergence associated to $J$ is denoted $D^{\theta}_J(\w, \w') \eqdef J(\w) - J(\w') - \langle \theta, \w - \w' \rangle$, where $\theta \in \partial J(\w')$.
The Fenchel-Legendre conjugate of $f$ is $f^\star(\theta) \eqdef \sup_w \langle w, \theta \rangle - f(w)$.
The indicator function $\iota_{\{\Y\}}$ is equal to zero if the argument equals $\Y$ and $+\infty$ otherwise.

\section{Over-parametrization, implicit and explicit regularization}

The basic problem of supervised learning is to  find a relationship to predict outputs $y$ from inputs $x$,
\begin{equation*}
    x \mapsto f(x) \approx y \enspace,
\end{equation*}
given a limited number of  pairs $(\x_i, \y_i)_{i=1}^n$ with, e.g. $\x_i\in \R^d$ and $\y_i\in \R$.
The search for a solution is typically restricted to  a set of parametrized functions $f_\w$,  with $w\in \R^p$.
A prototype example are linear models where $p=d$ and $f_\w(x) = \langle w, x \rangle$, or more generally $f_\w(x) = \sum_{j=1}^p \w^j  \phi_j(x)$, for some dictionary $\phi_j:\R^d\to \R,~~j=1, \dots, p$ \citep{Hastie_Tibshirani_Friedman09,Shalev-Shwartz_Ben-David14}.
In modern applications, it is often the case that the number of parameters $p$ is vastly larger than the number of available data points $n$, a regime called over-parametrized.
Excluding degenerate cases, one can then expect to find a solution $\w$ capable of interpolating the data, that is satisfying,
\begin{equation}\label{eq:interpolation}
    f_\w(x_i) = y_i, \quad \quad \forall i \in [n] \enspace.
\end{equation}
In the sequel we consider the case of a linear $f_w(x) = \langle w, x \rangle$.
A popular method to find a solution to \eqref{eq:interpolation} is gradient descent on least squares, also called Landweber iteration:
\begin{equation}\label{eq:gd_iterates}
    \w_{k}= w_{k-1} - \gamma \X^\top(\X w_{k-1} - \Y) \enspace,
\end{equation}
where $\X$ and $\Y$ are the data matrix and the outputs vector, respectively (see Section~\ref{Sec:3} for more details).
It is well known \citep{engl1996regularization} that, if initialized at $w_0 = 0$, the iterations of gradient descent converge to a specific solution, namely
\begin{problem}\label{pb:min_norm}
   \argmin_{w\in \R^p} \nor{w}  \quad \text{s.t.} \quad \X w = \Y \enspace.
\end{problem}
This means that amongst all  solutions, the algorithm  is  implicitly \emph{biased} towards that with small norm.
The bias is implicit in the sense that there is no explicit  penalization or constraint in the iterations \eqref{eq:gd_iterates}.
This approach can be contrasted to  explicit penalization  (Tikhonov regularization),
\begin{problem}\label{pb:tikhonov}
    w^{(\lambda)} = \argmin_{w\in \R^p} \lambda \nor{w}^2 + \nor{\Y - \X w}^2 \enspace,
\end{problem}
where the minimal norm solution \eqref{pb:min_norm} is  obtained for $\lambda$ going to zero.
It is well known that for Tikhonov regularization larger values of $\lambda$ improve stability.
Interestingly, the same effect can also be achieved with  gradient descent~\eqref{eq:gd_iterates},
by {\em not} running the iterations until convergence, a technique often referred to as early stopping \citep{engl1996regularization,yao2007early}.
In this view the number of iterations $k$ plays the role of a regularization parameter just like $\la$ in Tikhonov regularization (or rather $1/\lambda$).
Iterative regularization is particularly appealing in the large scale setting, where substantial computational savings are expected:
early stopping needs a finite number of iterations \eqref{eq:gd_iterates}, while  Tikhonov regularization requires solving
exactly \Cref{pb:tikhonov} for multiple values of $\lambda$.

It is natural  to ask whether the above iterative regularization scheme applies to biases beyond the Euclidean norm.
For a strongly convex $J$, an answer is given by considering the mirror descent algorithm \citep{nemirovsky1983problem,beck2003mirror} with respect to the Bregman divergence induced by $J$.
The bias $J$ is not used to define an explicit penalization of an empirical risk, but it appears in the mirror descent algorithm, and in this sense is ''less  implicit''.
The results in \citet{benning2016gradient} and \citet{gunasekar2018characterizing} show that mirror descent is implicitly biased towards the solution of the following problem
\begin{problem}\label{pb:generic_J}
    \argmin_{w\in \R^p} J(w)  \quad \text{s.t.} \quad \X w = \Y \enspace,
\end{problem}
and exhibit similar regularization and stability properties to the one of the gradient descent algorithm.
In both \citet{benning2016gradient} and \citet{gunasekar2018characterizing}, the key technical assumption is {\em strong convexity} of $J$ leaving open the question of dealing with biases that are only convex.
In this paper, we take steps to fill in this gap studying an efficient approach for which we characterize the iterative regularization properties.

\section{Problem setting and proposed algorithm}\label{Sec:3}

We begin describing the algorithm we consider and  its derivation. We first set some notation.
In the following, $\X$ is an $n$ by $p$ matrix and $\Y$ an $n$-dimensional vector.
Throughout, we assume that $n \leq p$ and that the linear equation has at least one solution for the exact data $\Y$.
For instance, if $\X$ has rank $n$, \Cref{pb:cvxbias} is feasible for every $\Y$.
In particular, a solution exists also for the noisy data $\Y^\delta$.
This is not the case in our general setting, where the solution to the noisy problem $\X w= \Y^\delta$ may not exist.

Note that, we use a vectorial notation for simplicity
but our results are dimension free and sharp for an infinite dimensional setting where $\X$ is a linear bounded
operator between separable Hilbert spaces.  In the following,  we also consider the case where
$\Y$  is  unknown, and a vector $\Y^\delta$ is available such that $\nor{\Y - \Y^\delta} \le \delta$,
where $\delta \ge 0$ can be interpreted as the noise level. We will assume the bias of interest to be specified by
a functional $J:\R^p\to \R\cup\{+\infty\}$ which is proper, convex and lower semicontinuous.

\subsection{Proposed algorithm}
Consider the following iterations, with initialization $\w_0 \in \R^p, \theta_0 = \theta_{-1} \in \R^n$, and parameters $\tau$, $\sigma$ such that $\sigma \tau \opnor{\X}^2 < 1$:
\begin{equation}\label{eq:cp}
    \begin{cases}
        \w_{k + 1} = \prox_{\tau \rr} (\w_k - \tau \X^\top (2\theta_k - \theta_{k - 1}))  \enspace,\\
        \theta_{k + 1} = \theta_k + \sigma (\X \w_{k + 1} - \Y) \enspace.
    \end{cases}
\end{equation}
If $\theta_0 = 0$, since $\theta_k = \sigma \sum_{1}^k (Xw_i - y)$, this algorithm can be rewritten without $\theta_k$:
\begin{equation*}
    w_{k + 1} = \prox_{\tau J} \Big(\w_k - \tau \sigma \X^\top \big( \textstyle\sum\nolimits_{1}^k (Xw_i - y) + Xw_k - y \big)\Big).
\end{equation*}
In terms of computations the algorithm \eqref{eq:cp} is very similar to the forward-backward/proximal gradient algorithm \citep{Combettes_Wajs2005}.
The difference is that the gradient term is here replaced by the sum of past gradients.
We instantiate algorithm \eqref{eq:cp} for two popular choices of $J$.
For $J = \nor{\cdot}^2$, the updates read:
\begin{equation*}
    \begin{cases}
        \w_{k + 1} = \frac{1}{1 + \tau}(\w_k - \tau \X^\top (2\theta_k - \theta_{k - 1})) \enspace, \\
        \theta_{k + 1} = \theta_k + \sigma (\X \w_{k + 1} - \Y) \enspace.
\end{cases}
\end{equation*}
Notice that, though involving very similar computations, the algorithm does not reduce to gradient descent iterations \eqref{eq:gd_iterates}.

For $J = \nor{\cdot}_1$, denoting by $\shrink (\cdot, \tau)$ the soft-thresholding operator of parameter $\tau$, the iterations \eqref{eq:cp} read:
\begin{equation*}
\begin{cases}
    \w_{k + 1} = \shrink(\w_k - \tau \X^\top (2\theta_k - \theta_{k - 1}), \tau) \enspace, \\
    \theta_{k + 1} = \theta_k + \sigma (\X \w_{k + 1} - \Y)   \enspace.
\end{cases}
\end{equation*}
Also in this case, it is similar -- yet not equivalent -- to a popular algorithm to solve the Tikhonov problem: the Iterative Soft-Thresholding Algorithm \citep{daubechies2004iterative}.

\begin{proposition}\label{prop:cp}
The iterations \eqref{eq:cp} converge to a point $(\ws, \thetas)$ such that $\X\ws = \Y$.
Additionally, $\ws$ is a minimizer of $J$ amongst all interpolating solutions, meaning that it solves
\begin{problem}\label{pb:cvxbias}
    \argmin_{w\in \R^p} J(w) \quad \text{s.t.} \quad \X w = \Y \enspace.
\end{problem}
\end{proposition}

We illustrate some examples of the above setting.



\bex[Sparse recovery]
Choosing $\rr = \nor{\cdot}_1$ corresponds to finding the minimal $\ell_1$-norm solution to a linear system, and in this case \Cref{pb:cvxbias} is known as Basis Pursuit \citep{Chen_Donoho_Saunders98}.
The relaxed approach of \Cref{pb:tikhonov} in this case yields the Lasso \citep{Tibshirani96}.
$\ell_1$-based approaches have had a tremendous impact in imaging, signal processing and machine learning in the last decades \citep{Hastie_Tibshirani_Wainwright15}.
\eex
\bex[Low rank matrix completion]\label{ex:lowrank}
In several  applications, such as recommendation systems, it is useful to recover a low rank matrix, starting from the observation of a subset of its entries \citep{Candes_Recht09}. A convex formulation is:
\begin{problem}\label{pb:low_rank}
    \min_{W \in \bbR^{p_1 \times p_2}} \nor{W}_* \quad \text{s.t. }  W_{ij} = Y_{ij}  \quad   \forall (i, j) \in \cD \enspace,
\end{problem}
where $\nor{\cdot}_*$ is the nuclear norm and $\cD \subset [p_1] \times [p_2]$ is the set of observed entries of the matrix $Y$.
In that case, $\X$ is a  self-adjoint linear operator from $\bbR^{p_1 \times p_2}$ to $\bbR^{p_1 \times p_2}$, such that $(X W)_{ij}$ has value $W_{ij}$ if $(i, j) \in \cD$ and 0 otherwise;
 the constraints write $\X W = \X Y$.
\eex

\bex[Total variation] In many imaging tasks such as deblurring and denoising, regularization through total variation allows to simultaneously  preserve edges whilst removing noise in flat regions \cite{Rudin_Osher_Fatemi92}.
The problem of Total Variation is
$\min_{W \in \R^{p_1 \times p_2}} \nor{\nabla W}_1 \text{s.t. } \X W = Y$,
and can be reformulated as:
    $\min_{\tilde W} \ \Omega(\tilde W)$ s.t. $\tilde\X \tilde W = \tilde Y$,
with $\tilde W = \begin{pmatrix} W \\ U \end{pmatrix},\  \Omega(\tilde W) = \Vert U \Vert_1, \ \tilde \X = \begin{pmatrix} \X  &0 \\ \nabla &-\mathrm{Id} \end{pmatrix}$ and $\tilde Y = \begin{pmatrix} Y \\ 0 \end{pmatrix}$.
\eex

\subsection{Chambolle-Pock algorithm}

In this section we prove \Cref{prop:cp} by casting \eqref{eq:cp} as an instance of the Chambolle-Pock algorithm \citep{Chambolle_Pock11} which solves:
$$\min_w f(\X w) + g (w) \enspace.$$
Hence, for $f = \iota_{\{\Y\}}$ and $g = J$, it can minimize a convex function on a set defined by linear equalities, as in \Cref{pb:cvxbias}.
For this choice of $f$ and $g$, it instantiates as \eqref{eq:cp} (see \Cref{app:cp}).

Amongst other assets, algorithm \eqref{eq:cp} only involves matrix-vector multiplications, and the computation of $\prox_J$, available in closed-form in many cases.
The only tunable parameters are two step-sizes, $\tau$ and $\sigma$, which are easy to set.
As usual for this class of methods, called primal-dual, the Lagrangian is a useful tool to establish convergence results.
The Lagrangian of \Cref{pb:cvxbias} is
\begin{equation}\label{eq:lagrangian}
    \cL(w, \theta) = J(w) + \langle \theta, \X w - \y \rangle \enspace,
\end{equation}
where $\theta \in \R^n$ is the dual variable.
Under a technical condition (\Cref{app:duality}), $\ws$ is a solution of \Cref{pb:cvxbias} if and only if there exists a dual variable $\thetas$ such that $\left(\ws,\thetas\right)$ is a saddle-point for the Lagrangian, namely, iff for every $\left(\w, \theta\right) \in \R^p \times \R^n$,
\begin{equation}\label{optcond2}
\cL(\ws, \theta)\leq\cL(\ws, \thetas)\leq \cL(\w, \thetas) \enspace.
\end{equation}
The variable $\theta$ is in our setting just an auxiliary variable, and we will be interested in convergence properties of $w_k$ towards $\ws$.

{\bf  Other algorithms}  As mentioned in the introduction,  other algorithms could be considered, e.g. ADMM/Bregman iteration.
However, we are not aware of  methods that can be efficiently implemented in our general setting. In \Cref{sec:app_related_works}.
we provide an extensive review discussing the connection with a number of different approaches and related works.

\section{Theoretical analysis}
\label{sec:theoretical}

In this section, we analyze the convergence properties of Algorithm \eqref{eq:cp}.
First, we need to choose a suitable criterion to estimate the approximation properties of the iterates.
In general, it is not reasonable to expect a rate of convergence for the distance between the iterates and the
solution.
Indeed, since the problem is only convex, it is well known that the convergence in distance can be arbitrarily slow.
In \Cref{metricsec}, we explain why a reasonable choice is given by the duality gap together with the residual norm (respectively, $\cL(\w_k, \thetas)-\cL(\ws, \theta_k)$ and $\nor{\X\w_k-\Y}$).
For these two quantities, we derive:
\vspace*{-2mm}
\begin{itemize}{\topsep=0pt}
    \item
    convergence rates in the exact case, i.e. when the data $\Y$ is available (\Cref{rate});
    \item
     early-stopping bounds in the inexact case, i.e. when the accessible data is only $\Y^\delta$ with $\nor{\Y^{\delta}-\Y}\leq \delta$ (\Cref{early} and \Cref{earlstop}).
\end{itemize}

In \Cref{subsec:l1}, we apply our analysis to the specific choice of $J$ equal to the $\ell_1$-norm.
In this particular case, combining the previous results, we even obtain bounds directly on $\nor{\w_k - \ws}$.

\subsection{Measure of optimality}\label{metricsec}
To discuss which criterion is significant to study the algorithm convergence, we recall from \eqref{optcond2} that, if
\begin{equation}\label{eq:saddle}
\cL(\w', \theta)-\cL(w, \theta')\leq 0
\end{equation}
for every $\left(\w,\theta\right)\in\R^p\times\R^n$, then $\left(w',\theta'\right)$ is a primal-dual solution.
In general, it is difficult to prove that \Cref{eq:saddle} holds for every $\left(\w,\theta\right)\in\R^p\times\R^n$.
Then, given a saddle-point $(\ws, \thetas)$ and a generic $(\w', \theta')\in\R^p\times\R^n$, it is popular to consider the  quantity
\begin{equation}\label{eq:pdgap}
\cL(\w', \thetas) - \cL(\ws, \theta') \geq 0 \enspace.
\end{equation}
To establish the optimality of $\left(w',\theta'\right)$, it is not enough to ensure
$\cL(\w', \thetas)-\cL(\ws, \theta') = 0.$
\Cref{lem}, proved in \Cref{sub:app_lemmas}, shows that this condition, when coupled with $\X\w'=\Y$, implies that $\w'$ is a solution of \cref{pb:cvxbias}.
%
\begin{restatable}{lemma}{optimalitylemma}\label{lem}
	Let $(\ws, \thetas)$ be a primal-dual solution and $(\w', \theta')$ a point in $\R^p\times\R^n$ such that
	$\cL(\w', \thetas)-\cL(\ws, \theta')=0$ and $\X\w' = \Y$.
	Then $(\w', \thetas)$ is a primal-dual solution.
\end{restatable}
Thus, the  quantities $\cL(\w_k, \thetas)-\cL(\ws, \theta_k)$ and $\nor{\X\w_k-\Y}$, studied together, are a reasonable measure of optimality for  the iterate $\w_k$.

Note that $\cL(\w_k, \thetas)-\cL(\ws, \theta_k)$ is the error measure used in a series of papers dealing with regularization of inverse problems with general convex regularizers, see e.g. \cite{burger2007error}.
It is well known that if $J$ is strongly convex then this quantity controls the distance in norm (\Cref{rem:control_stronglycvx}) and therefore is a proper measure of convergence.
If $J$ is only convex, this measure of error can be quite weak.
In \Cref{subsec:l1} we point out the limitations of this quantity when dealing with  $J= \nor{\cdot}_1$.
For this choice of $J$, $\cL(0, \thetas)-\cL(\ws, \theta_k)$ is 0 for any $\theta_k$; as shown on \Cref{fig:bregman_sparse}, this quantity is also 0 when $w_k$ and $\ws$ have the same support and sign.
\begin{figure*}
	\centering
	\includegraphics[width=\linewidth]{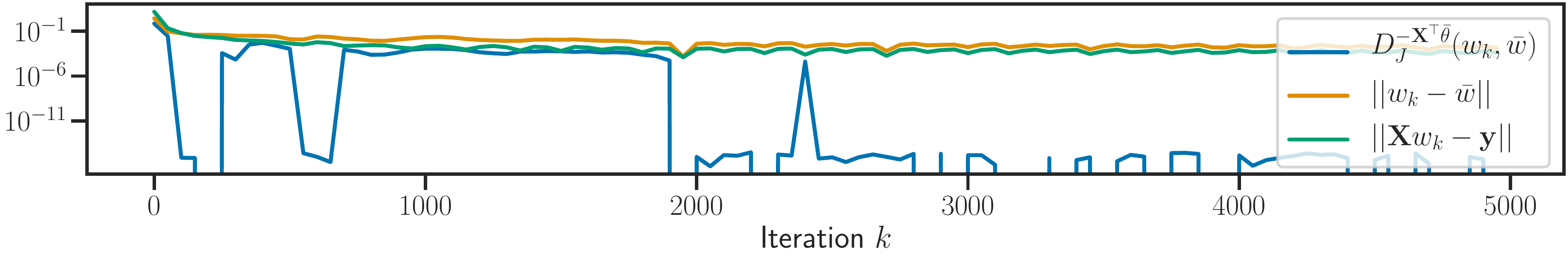}
	\captionof{figure}{For $J = \nor{\cdot}_1$, convergence of primal-dual iterates $\w_k$ towards $\ws$, measured in norm, feasability and Bregman divergence.
		The Bregman divergence quickly vanishes up to numerical errors (the iterates have the same sign as the solution); yet the iterates are still far from the solution.}
	\label{fig:bregman_sparse}
\end{figure*}%
\subsection{Exact case}\label{exactsec}
First consider the iterates $(\w_k, \theta_k)$ obtained by applying iterations \eqref{eq:cp} to the exact problem, namely where the data $\Y$ is available. Let $(\ws, \thetas)$ be a saddle-point for the Lagrangian.
Denoting the primal-dual variables by $z = (\w, \theta)$, we have $z_k=\left(\w_k,\theta_k\right)$ for the iterates of the algorithm and $\zs=\left(\ws,\thetas\right)$ for the saddle-point.
For $\tau$ and $\sigma>0$, define $V$ as the following square weighted norm on $\R^p\times \R^n$:
\begin{equation}\label{VVV}
V(z) \eqdef \frac{1}{2\tau} \nor{\w}^2 + \frac{1}{2\sigma} \nor{\theta}^2. \enspace
\end{equation}
For the averaged iterates
$    {\overline\w}^k  \eqdef \frac{1}{k}\sum_{t=1}^{k} \w_t$ and ${\overline\theta}^k  \eqdef \frac{1}{k}\sum_{t=1}^{k} \theta_t$,
we have the following rates.

\begin{restatable}[Convergence rates]{proposition}{earlystoppinggap_exact}\label{rate}
	Let $\varepsilon \in(0,1)$ and assume that  the step-sizes are such that $\sigma \tau \leq \varepsilon/\opnor{\X}^{2}$. Then
	\begin{align*}
    &\cL({\overline\w}^k, \thetas) - \cL(\ws, {\overline\theta}^k) \ \leq \ \frac{\sqrt{V(z_0 - \zs)}}{k}\quad\quad  \text{and} \\
    & \nor{\X {\overline\w}^k - \Y}^2 \leq \frac{2 (1 + \varepsilon) V(z_0 - \zs)}{\sigma\varepsilon (1 - \varepsilon) k} \enspace.
	\end{align*}
\end{restatable}
The first result is classical (see \cite{Chambolle_Pock11}).
Alternatively, it can be obtained by setting $\delta=0$ in \Cref{early}, where we study the more general inexact case.
To the best of our knowledge, the second bound is new and can also be derived by setting $\delta=0$ in \Cref{early}.
A similar result, in the more specific case of primal-dual coordinate descent, can be found in \citet{Fercoq_Bianchi19}.
Note that both results of \Cref{rate} are true for every primal-dual solution.
On the other hand, the left-hand-side in the second equation does not depend on the selection of $\zs$ and so the bound can be improved by taking the $\inf$ over all primal-dual solutions.

\subsection{Inexact case}\label{inexactsec}
We now consider the iterates $(\w_k, \theta_k)$, and their averaged versions $(\overline{\w}_k, \overline{\theta}_k)$, obtained by applying iterations \eqref{eq:cp} to the noisy problem, where $\Y$ is replaced by $\Y^\delta$ with $\nor{\Y^{\delta}-\Y}\leq \delta$.
In \Cref{early}, we derive early-stopping bounds for the iterates, in terms of duality gap $\cL(\w_k, \thetas)-\cL(\ws, \theta_k)$ and residual norm $\nor{\X\w_k-\Y}$.
We highlight that, despite the error in the data $\Y^{\delta}$, both quantities are defined in terms of $\Y$ and hence related to the \emph{noiseless} problem.
In particular, $(\ws, \thetas)$ is a saddle-point for the \emph{noiseless} Lagrangian. 
We have the following estimates, whose proofs are given in \Cref{sec_early}.

\begin{restatable}[Stability]{proposition}{earlystoppinggap}\label{early}
	Let $\varepsilon \in(0,1)$ and assume that  the step-sizes are such that $\sigma \tau \leq \varepsilon/\opnor{\X}^{2}$. Then,
	\begin{equation}
	\begin{split}
	\cL({\overline\w}^k, \thetas) - \cL(\ws, {\overline\theta}^k) & \leq \tfrac{1}{k}\left(\sqrt{V(z_0 - \zs)}  +  \sqrt{2\sigma} \delta k\right)^2
	\end{split}
	\end{equation}
	and
    \begin{equation}\label{est_feas}
	\begin{split}
    &\nor{\X {\overline\w}^k - \Y}^2
    \leq \frac{2(1 + \varepsilon)}{\sigma \varepsilon (1 - \varepsilon)} \Bigg[ \sqrt{2\sigma V(z_0 - \zs)}\delta
    \\ &\hspace*{1.5cm}  + \frac{\sigma \varepsilon}{1 - \varepsilon}\delta^2 + 2\sigma \delta^2 k + \frac{1}{k} V(z_0 - \zs) \Bigg].
	\end{split}
	\end{equation}
\end{restatable}

Note that, in the exact case $\delta=0$, we recover the convergence results stated in \Cref{rate}.
Moreover, we have the following corollary.
\begin{corollary}[Early-stopping]\label{earlstop}
	Under the assumptions of \Cref{early}, choose $k = c/\delta$ for some $c>0$.
	Then there exist constants $C, \ C'$ and $C''$ such that
	\begin{align*}
	&\cL({\overline\w}^k, \thetas) - \cL(\ws, {\overline\theta}^k) \ \leq \ C \delta \text{ and}\\
	&\nor{\X {\overline\w}^k - \Y}^2 \leq C' \delta + C'' \delta^2 \enspace.
	\end{align*}
\end{corollary}
The constants appearing in the Corollary are the ones from \Cref{early}.
They only depend on the saddle-point $\zs$, the initialization $z_0$ and the step-sizes $\tau, \sigma$.
We next add some remarks.
	\begin{remark}\label{rem:control_stronglycvx}
		When $\rr$ is $\gamma$-strongly convex, in particular when $J(\cdot)=\frac12 \nor{\cdot}^2$, both the residual norm and the distance between the averaged iterate and the solution can be controlled by $\cL({\overline\w}^k, \thetas) - \cL(\ws, {\overline\theta}^k)$. Indeed, recalling \Cref{breggap},
		\begin{align*}
        \nor{\X {\overline\w}^k - \Y}^2
        &\leq\nor{\X}^2\nor{{\overline\w}^k-\ws}^2
        \\ &\leq \tfrac{2\nor{\X}^2}{\gamma} D_J^{-\X^\top\thetas}\left({\overline\w}^k,\ws\right) \\
        &= \tfrac{2\nor{\X}^2}{\gamma}\left[\cL({\overline\w}^k, \thetas) - \cL(\ws, {\overline\theta}^k)\right] \enspace.
		\end{align*}
		In particular, the previous early-stopping bounds are of the same order of the ones obtained by dual gradient descent in \citet{matet2017don}.
	\end{remark}
	\begin{remark}
		Similar estimates have been obtained in \citet{burger2007error}, both for the Tikhonov variational scheme and for the Bregman iteration (also called inverse scale space method) with stopping-criteria given by the discrepancy principle.
        In the first case (see Theorem 3.1), for a suitable choice of the regularization 		parameter, the authors get similar estimates for the Tikhonov regularized solution $w_\lambda$:  $D_J^{s}\left(\w_{\lambda},\ws\right) \leq C\delta$ and
		$\nor{\X\w_{\lambda}-\Y}^2\leq C'\delta^2$, where $D^{s}_J$ is the symmetric Bregman
		divergence. For the Bregman iteration (see Theorem 4.2), they get an early-stopping bound on $D_J^{p_k}(\ws,\w_k)$, where $p_k\in\partial J(\w_k).$
		Note that they do not get any estimate for the quantity $D_J^{-\X^\top\thetas}\left(\w^k,\ws\right)$ neither for the residual norm. Moreover, the method requires to solve, at each iteration, an optimization problem with the same complexity of the original one.
	\end{remark}

{\bf Proof Sketch\ }  The proof of \Cref{early} is inspired by \citet{rasch2020inexact}.
In this paper, the kind of errors allowed in the prox of the non-extrapolated step ($\theta_k$ update) are more general than the ones allowed for the extrapolated step ($w_k$ update).
Here, we study stability properties of algorithm \eqref{eq:cp} when $\Y$ is replaced by $\Y^\delta$.
This change can be read as an inexact proximity operator in the update of $\theta$ computation; in order to have this error in the non-extrapolated step, we study algorithm \eqref{eq:cp}, that is CP algorithm applied to the dual problem.
We summarize here the main steps.
In \Cref{onestep}, we derive a ``descent property'' for every step $t$,  which we then cumulate summing from $t=1$ to $t=k$ and using two different approximations (\Cref{cum1,cum2}).
The two bounds that we get are similar, but independent.
The first one has the following form,
\begin{align}\label{eq:cumulative}
&\frac{1}{2\sigma}\nor{\theta_{k} - \thetas}^2 + \sum_{t=1}^{k}\left[\mathcal{L}(\w_t,\thetas) - \mathcal{L}(\ws, \theta_t)\right]
\leq \nonumber \\ & \hspace*{2cm} V(z_0 - \zs) +  \delta \sum_{t=1}^{k}\nor{\theta_t - \thetas} \enspace.
\end{align}

We use the latter twice.
First we combine it with \Cref{lem:u_n_upper_bound}, a discrete version of Bihari's Lemma. This allows to estimate, for every $1\leq t\leq k$, the quantity
\begin{align}
\nor{\theta_t - \thetas}   \leq 2\sigma\delta k + \sqrt{2\sigma V(z_0 - \zs)} \enspace.
\end{align}
Then we use again \Cref{eq:cumulative}, joint with the previous information, to find a bound on $\sum_{t=1}^{k}\left[\mathcal{L}(\w_t,\thetas) - \mathcal{L}(\ws, \theta_t)\right]$.
The second inequality (see \Cref{cum2}) has the following form,
\begin{align}
\frac{\sigma\alpha}{2\eta}\sum_{t=1}^{k}\nor{\X \w_{t}- \Y}^2  &\leq  V(z_0 - \zs)  + \delta \sum_{t=1}^{k}\nor{\theta_t - \theta} \nonumber\\
& \hspace*{0.3cm} + \frac{1}{2}\sigma \left(\eta-1\right)\delta^2k\enspace.
\end{align}
Using again the bound on $\nor{\theta_t - \thetas}$ and choosing $\eta = \left(1 + \varepsilon\right)/\left(1 - \varepsilon\right)$, we find an estimate for $\sum_{t=1}^{k}\nor{\X \w_{t}- \Y}^2$.
In both cases, we get the claim on the averaged iterates by Jensen's inequality.

\subsection{An example: sparse recovery}
\label{subsec:l1}

In the case of sparse recovery ($\rr = \nor{\cdot}_1$), controlling the duality gap and the feasability yields a bound on the distance to the minimizer, thanks to the following result \cite[Lemma 3.10]{grasmair2011necessary}.

\begin{lemma}\label{lem:grasmair}
	Let $(\ws, \thetas)$ be such that $\X\ws = \Y$ and $- \X^\top\thetas \in \partial \nor{\cdot}_1 (\ws)$.
	With $\Gamma \eqdef \{j \in [p]: \abs{\X_{:j}^\top \thetas} = 1 \}$, assume that $\X_\Gamma$ ($\X$ restricted to columns whose indices lie in $\Gamma$) is injective.
	Let $m \eqdef \max_{j \notin \Gamma} \abs{\X_{:j}^\top \thetas} < 1$.
	Then, for all $\w\in\R^p$,
	\begin{align}
	\nor{\w - \ws}
    &\leq \opnor{\X^{-1}_\Gamma} \nor{\X \w - \Y} \nonumber\\
    &\hspace*{0mm} + \tfrac{1 + \opnor{\X^{-1}_\Gamma} \opnor{\X}}{1 - m} D^{-\X^\top\thetas}_{\nor{\cdot}_1}(\w, \ws) \enspace.
	\end{align}
\end{lemma}
Note that, under the assumptions of \Cref{lem:grasmair}, the primal solution to \Cref{eq:prob} is unique (see \cite[Thm 4.7]{grasmair2011necessary}). Combining the latter with \Cref{earlstop} yields a strong early-stopping result.
\begin{corollary}[Early-stopping for $J = \nor{\cdot}_1$]\label{earlstop_sparse_recov}
	Under the assumptions of \Cref{early,lem:grasmair}, choose $k = c/\delta$ for $c>0$.
	Then there exist constants $C'$ and  $C''$ such that
	\begin{equation*}
	\nor{\overline{\w}^k - \ws} \leq C'\sqrt{ \delta} + C'' \delta \enspace.
	\end{equation*}
\end{corollary}
The constants $C', C''$ depend on the saddle-point $\zs$, the initialization $z_0$, the step-sizes $\tau, \sigma$ and the norms of $\X$ and $\X^{-1}_\Gamma$.
A completely different approach has been considered, for the same problem, in \citet{Vaskevicius_Kanade_Rebeschini19}. A related  approach, based on dynamical systems, has been proposed in \citet{osher2016sparse}. Similar results for the Tikhonov regularization approach can be found in \citet{schlor19}.

\begin{figure}[t]
    \centering
    \includegraphics[width=0.5\linewidth]{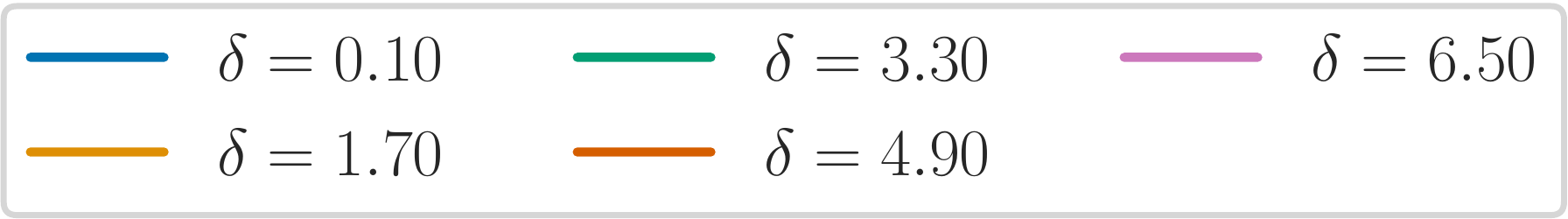}
    \includegraphics[width=0.5\linewidth]{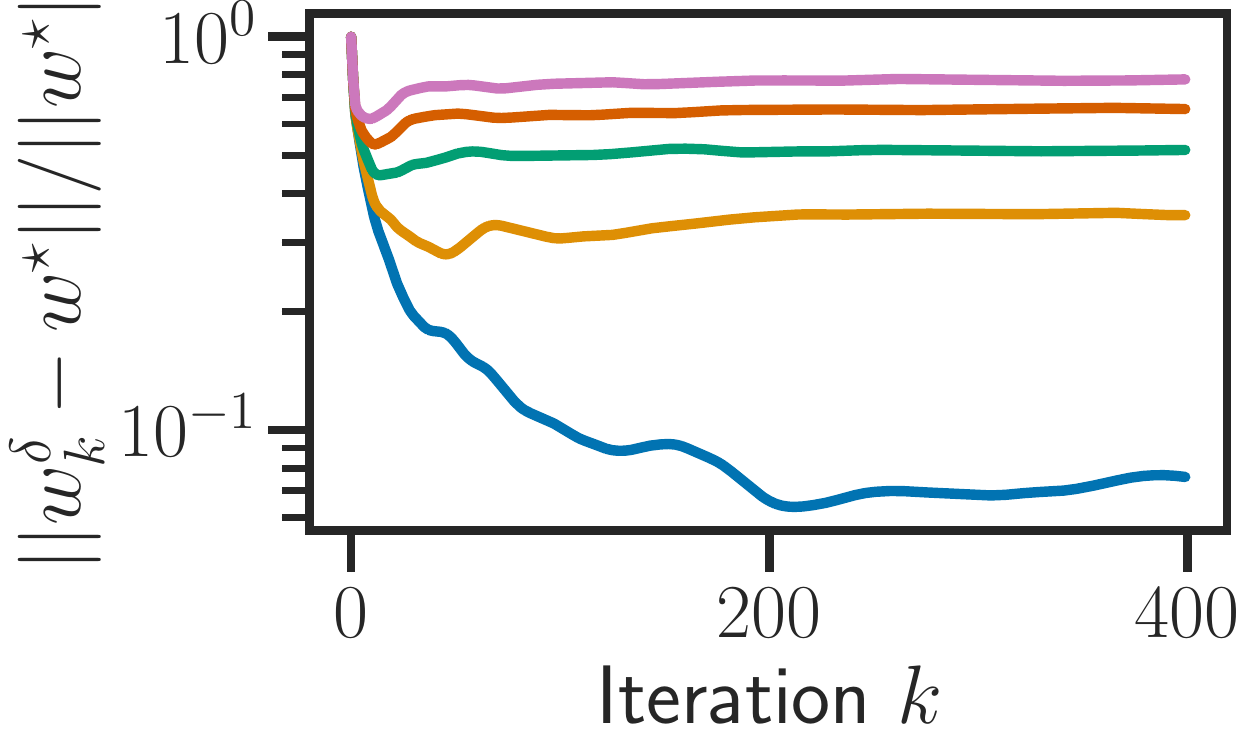}
    \caption{Distance between noisy Basis Pursuit iterates $\w_k^\delta$ and noiseless solution $\ws$, for various values of $\delta$.
    There exists a stopping time: these distances reach a minimum before converging to their limit.}
    \label{fig:existence_stopping_time}
\end{figure}

\section{Empirical analysis}
We stress that there is no implicit regularization result dealing with any non strongly convex $J$ to compare to.
The only competitor is therefore the Tikhonov approach.
The code is made available in the supplementary material as a python package with scripts to reproduce the experiments (relying heavily on numpy \citep{harris2020array} and numba \citep{numba}).


\subsection{Sparse recovery}
Random data for this experiment are generated as follows: $(n, p) = (200, 500)$, columns of $\X$ are Gaussian with $\cov(\X_{:i}, \X_{:j}) = 0.2^{|i - j|}$, $\Y = \X w_0$ where $w_0$ has 75 equal non zero entries, scaled such that $\nor{\Y} = 20$ (in order to have a meaningful range of values for $\delta$).
Note that the linear system $\X\w=\Y^\delta$ has solutions for any $\Y^\delta$, since $\X$ is full-rank.
The noiseless solution $\ws$ is determined by running algorithm \eqref{eq:cp} up to convergence, on $\Y$.
For the considered values of $\delta$, $\Y^\delta$ is created by adding i.i.d. Gaussian noise to $\Y$, so that $\nor{\Y - \Y^\delta} = \delta$.
We denote by $\w_k^\delta$ the iterates of algorithm \eqref{eq:cp} ran on $\Y^\delta$.

{\bf Existence of stopping time.}
In the first experiment, we highlight the existence of an optimal iterate in terms of distance to $\ws$.
\Cref{fig:existence_stopping_time} shows \emph{semi-convergence}: before converging to their limit, the iterates get close to $\ws$.
Note that this is stronger than the results of \Cref{earlstop_sparse_recov}, since the optimal iteration here is the minimizer of the distance, and not some iterate for which there exists an upper bound on the distance to $w^\star$.
As expected, as $\delta$ decreases, the optimal iteration $k$ increases and the optimal iterate $w_k^\delta$ is closer to $\ws$.

{\bf Dependency of empirical stopping time on $\delta$.}
In the same setting as above, for 20 values of $\delta$ between 0.1 and 6, we generate $100$ values of $\Y^\delta.$
We run algorithm \eqref{eq:cp} for 5000 iterations on $\Y^\delta$ and determine the empirical best stopping timel as $k^\star(\delta) = \argmin_k \nor{\w_k^\delta - \ws} < +\infty$.
\Cref{fig:stopping_time_delta} shows the mean of the inverse empirical stopping time as a function of $\delta$, where a clear linear trend ($k = c / \delta$) appears as suggested by \Cref{early} and \Cref{earlstop_sparse_recov}.

\begin{figure}[t]
    \centering
    \includegraphics[width=0.6\linewidth]{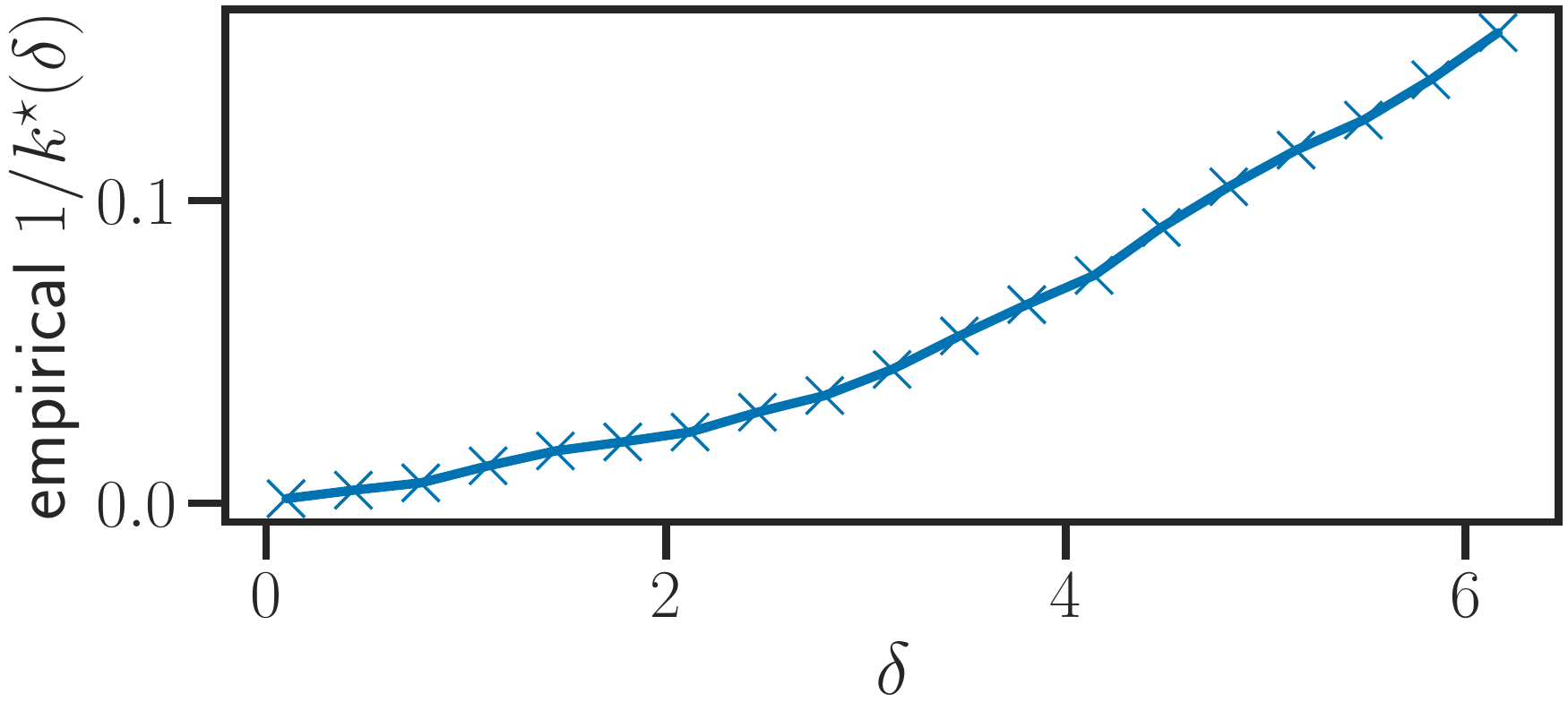}
    \caption{Influence of $\delta$ on the stopping time. In agreement with theory, the empirical stopping time roughly scales as $1 / \delta$.
    }
    \label{fig:stopping_time_delta}
\end{figure}

{\bf Comparison with Tikhonov approach on real data.}
The most popular approach to address stability is to solve \Cref{pb:tikhonov} (here, the Lasso) for, typically\footnote{default grid in scikit-learn \citep{Pedregosa_etal11} and GLMNET \citep{Friedman_Hastie_Tibshirani10} packages} 100 values of $\lambda$ geometrically chosen as $\lambda_t = 10^{-3t/99}\nor{\X^\top \Y}_\infty$ for $t=0, \ldots, 99$.
In \Cref{fig:cv_vs_cp} we compare the Lasso regularization path to the Basis Pursuit optimization path of the Chambolle-Pock algorithm.
The dataset for this experiment is rcv1-train from libsvm \citep{Fan_Chang_Hsieh_Wang_Lin08}, with $(n, p) = (\num{20242}, \num{26683})$.
The figure of merit is the mean squared error on left out data, using 4-fold cross validation (dashed color lines), with the average across the folds in black.
The horizontal line marks the $\lambda$ (resp. the iteration $k$) for which the Lasso path (resp. the optimization path of Algorithm \eqref{eq:cp}) reaches its minimum MSE on the test fold.

The first observation is that the Basis Pursuit solution (both the end of the optimization ($k = +\infty$) and regularization paths ($\lambda = 0$)) performs very poorly, having a MSE greater than the one obtained by the 0 solution For the bottom plot, this would also be visible if the number of iterations of Algorithm \eqref{eq:cp} was picked greater than 500, which we do not do for readability of the figure.
It is therefore necessary to early stop.
The second observation is that the minimal MSEs on both paths are similar: 0.19 for Lasso path, 0.21 for optimization path of Algorithm \eqref{eq:cp}.
The main point is however that it takes 20 iterations of algorithm \eqref{eq:cp} to reach its best iterate, while the optimal  $\lambda$ for the Lasso is around $\lambda_{\max} / 100$.
If the default grid of 100 values between $\lambda_{\max}$ and $\lambda_{\max} / 1000$ was used, this means that 66 Lassos must be solved, each one needing hundreds or thousands of iterations to converge.
This is reflected in the timings: 0.5~s for Algorithm \eqref{eq:cp} vs 50~s for Tikhonov, eventhough we use a state-of-the-art coordinate descent + working set approach to solve the Lasso, with warm-start (using the solution for $\lambda_{t-1}$ as initialization for problem with $\lambda_t$).

\begin{figure}[t]
    \centering
    \includegraphics[width=0.5\linewidth]{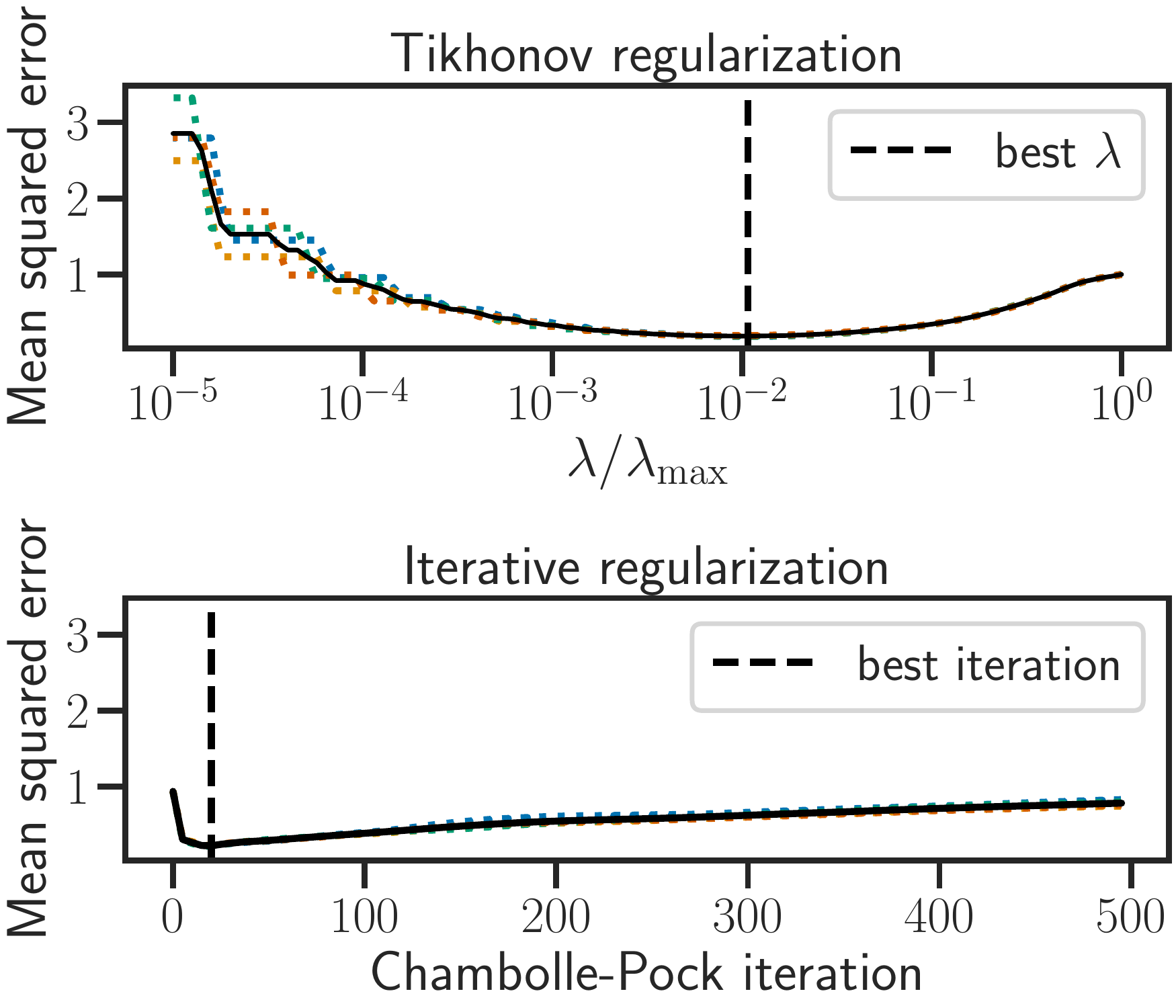}
    \caption{Comparison of Tikhonov regularization path (top) and optimization path of Algorithm \eqref{eq:cp}
     (bottom) on rcv1 with 4-fold cross-validation. Minimal value reached: 0.19 (top), 0.21 (bottom). Computation time up to optimal parameter: 50 s (top); 0.5 s (bottom).}
    \label{fig:cv_vs_cp}
\end{figure}

\subsection{Low rank matrix completion}

Random data for this experiment is generated as follows: the matrices are $d \times d$ with $d = 20$.
$\mathbf Y$ is equal to $U V^\top$ with $U$ and $V$ of size $d \times 5$, whose entries are i.i.d Gaussian ($\mathbf Y$ is rank 5).
We scale $\mathbf Y$ such that $\nor{\mathbf{Y}} = 20$.
Recall that in low rank matrix completion (\Cref{ex:lowrank}), $\X$ corresponds to a masking operator (the observed entries); to determine which entries are observed, we uniformly draw $d^2 / 5$ observed couples $(i, j) \in [d] \times [d]$.
\Cref{fig:semiconv_lowrank} shows the same type of results as \Cref{fig:existence_stopping_time}: iterates first approach the noiseless solution, then get further away, justifying early stopping of the iterates.
For this experiment, we use higher values for $\delta$ to better highlight the semiconvergence, as curves get flatter for e.g. $\delta = 5.7$.
Note that in that case, the algorithm can still be early stopped to save computations.

\begin{figure}[tbp]
    \centering
    \includegraphics[width=0.5\linewidth]{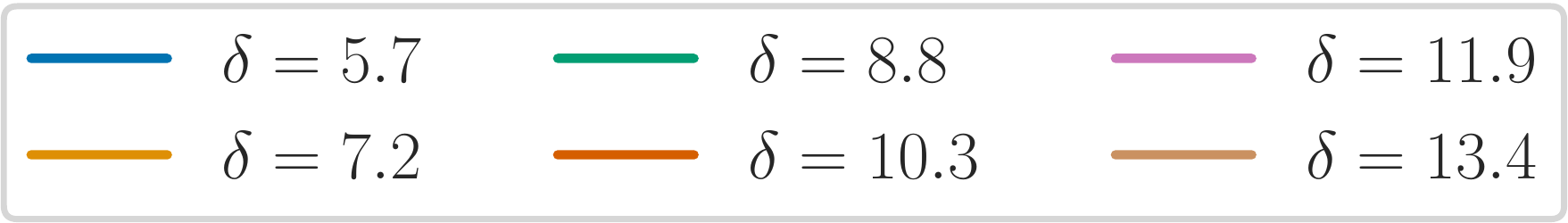}
    \includegraphics[width=0.5\linewidth]{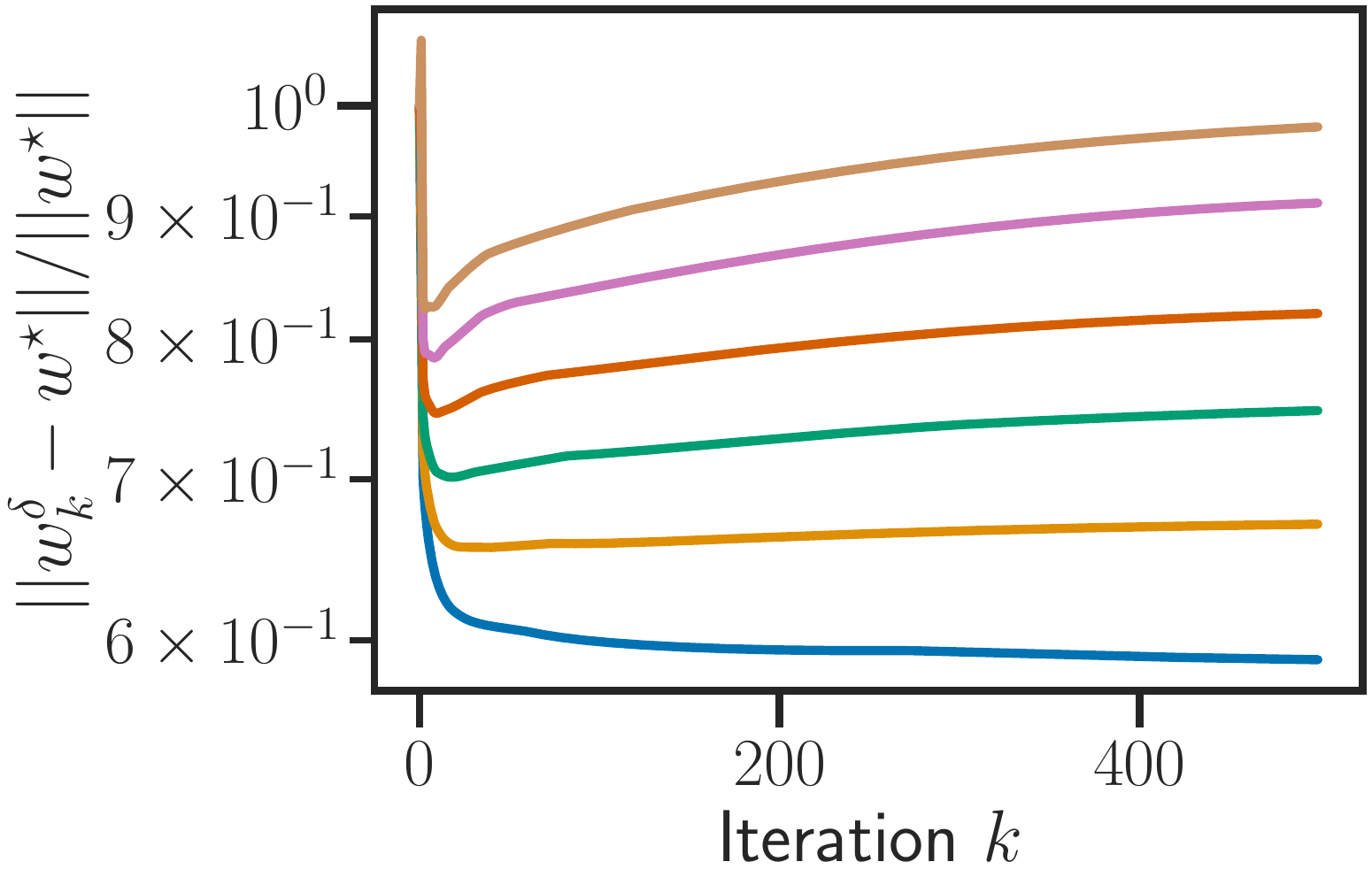}
    \caption{Distance between low rank matrix completion iterates $\w_k^\delta$ and noiseless solution $\ws$, for various values of $\delta$.
    There exists a stopping time: a minimum before distance is reached before the limit.}
    \label{fig:semiconv_lowrank}
\end{figure}

\section{Conclusion}
We have studied implicit regularization for convex bias, not necessarily strongly convex nor smooth.
We proposed to use the Chambolle-Pock algorithm and we analyzed both convergence and stability to deterministic worst case noise.
Our general analysis was specialized, as an example, to the problem of sparse recovery.
The approach was  investigated empirically both for sparse recovery and matrix completion, showing great timing improvements over relaxation approaches.
A future development is to consider more specific noise models than the worst-case, such as stochastic noise.
We emphasize again that our results hold in infinite dimension.
It would be interesting to specialize our analysis in the finite dimensional setting, when the noisy solution always exists (in the least-square sense) and so the iterates produced by the algorithm are bounded.
Moreover, it would be interesting to consider additional assumptions such as sparsity.
Considering the role of initialization or nonlinear models would also be of interest.
Finally, it would complete the analysis to obtain lower bounds for this class of problems, to confirm the sharpness of our results.

\newpage
\section*{Acknowledgments}
This material is based upon work supported by the Center for Brains, Minds and Machines (CBMM), funded by NSF STC award CCF-1231216, and the Italian Institute of Technology.
Part of this work has been carried out at the Machine Learning Genoa (MaLGa) center, Universit\`a di Genova (IT).
L. R. acknowledges the financial support of the European Research Council (grant SLING 819789), the AFOSR projects FA9550-17-1-0390  and BAA-AFRL-AFOSR-2016-0007 (European Office of Aerospace Research and Development), and the EU H2020-MSCA-RISE project NoMADS - DLV-777826.
S. V. acknowledges the support of INDAM-GNAMPA, Project 2019: ``Equazioni integro-differenziali:
aspetti teorici e applicazioni''.

\bibliographystyle{plainnat}

\newpage
\input{./appendix}

\end{document}

%% file: appendix.tex
\onecolumn
\appendix

\section{Detailed discussion of related works}\label{sec:app_related_works}

The idea of exploiting the implicit regularizing properties of optimization algorithms is not new, and has been studied in three different related areas, often under the name of iterative regularization: inverse problems, image restoration, and machine learning.
The related results can roughly be divided in
those assuming strong convexity of $J$ and those assuming only convexity of $J$.
Related approaches to implicit regularization include  diagonal strategies and exact regularization approaches.
Extensions to general data fits and non-convex/non-linear problems have been considered. In the following we briefly review existing results.

{\bf $\bullet$ Gradient and stochastic descent.} The study of implicit regularization properties of gradient descent, known in the inverse problem community as Landweber method,  goes back to the 50s \citep{engl1996regularization}. The classical result shows that gradient descent applied to least squares and initialized at $0$ converges to the minimal norm solution of the linear equation \eqref{eq:interpolation}. Accelerated versions have been also studied under the name of $\nu$-method \citep{engl1996regularization}.
Generalization towards more general regularizers, apart from $p$ norms with $p>1$, has not been considered much by this community,
while there is a rich literature in the non-convex setting for nonlinear inverse problems \citep{kaltenbacher2008iterative}.
These ideas have been extended to machine learning  considering regularizing properties of gradient descent \citep{yao2007early} and its stochastic versions  \citep{moulines2011non,rosasco2015learning}.

{\bf $\bullet$ Linearized Bregman and Mirror descent.} The interest in more general regularizers has been mainly motivated by imaging applications and total variation regularization.  Starting from \cite{OshBurGol05}
there is an entire line of work devoted to iterative regularization  for general convex regularizers (see e.g. \cite{burger2007error} and references therein). We briefly review the available algorithms and their advantages and limitations.
If {strong convexity of $J$ is assumed}, the algorithm of choice is mirror descent \citep{Nemirovski_Yudin83,BecTeb03}.
It has been popularized in the inverse/imaging problems community
under the name of linearized Bregman iteration \citep{YinOshBur08}. It has been shown that  this algorithm in combination with a discrepancy type stopping rule
regularizes ill posed problems. The stability and regularization properties of the accelerated variant of the algorithm have been studied also in \citet{matet2017don}, using a different approach,  based on the interpretation of the method as a gradient descent applied to the dual problem \eqref{pb:dual2}. Similar ideas  can be found in \citep{SchLor18}.

{\bf $\bullet$ Bregman iteration and ADMM.}  If the regularizer $J$ is not strongly convex, but only convex, as in our case, the algorithm above cannot be applied.
The algorithm of choice is in this context ADMM \cite{BoyParChu10}, which  has been studied in the imaging community under the name of Bregman iteration.
Its regularization properties can be found in \cite{burger2007error}. However, this method has a main  drawback: at each iteration the solution of a nontrivial optimization problem of the form  $\min \|\X \w-\Y\|^2 +\rr(\w)+\langle \w,\eta\rangle$, for $\eta\in\R^p$ is required, and in general la subroutine is needed at each iteration.
In the setting where $n$ is big, this can  have a high computational cost.
The extension of this approach to nonlinear inverse problems has been considered in \citet{bachmayr2009iterative}.

{\bf $\bullet$ Bregmanized Operator Splitting and linearized/preconditioned ADMM.} These are variants of Bregman iteration and ADMM very similar to the CP algorithm: they rely on preconditioning to avoid the solution of a difficult optimization problem at each iteration. These have been used empirically as regularizing procedures in the context of inverse and imaging problems \citep{ZhaBurOsh11,ZhaBurBre10}. We are not aware of any theoretical quantitative stability result.

{\bf $\bullet$ Diagonal approaches} The implicit regularization techniques described above are well-suited for problems where the quadratic data fit is appropriate.
If other losses are used, this approach completely neglect them. A way to circumvent this problem is to use a diagonal strategy. The idea is to combine an
optimization algorithm with a sequence of approximations of the original problem \eqref{eq:prob} which change at each iteration \citep{BahLem94}.
 Convergence rates and stability of diagonal approaches for inverse problems have been considered in \citet{garrigos2018iterative,calatroni2019accelerated}.

{\bf $\bullet$ Sparse recovery and compressed sensing} In the context of sparse recovery the implicit regularization approach has been considered in
\cite{osher2016sparse}, and also in \cite{Vaskevicius_Kanade_Rebeschini19}.
Matching pursuit \cite{Mallat_Zhang93} is a computational procedure which can be used to select relevant components, but
it is not clear from the theoretical point of view how to early stop the iterations \cite{}.

{\bf $\bullet$ Exact regularization}
Another possible approach is to use the notion of exact regularization \citep{friedlander2008exact,schopfer2012exact}.
The latter refers to solving
$  \min_{\w} \rr(\w) + \alpha Q(\w), \quad \text{s.t.} \quad \X \w = \Y \enspace,$
where $Q$ is strongly convex and to showing that there exists a value of $\alpha$ such that this new problem and \Cref{pb:cvxbias} have the same minimizer.
Then, known iterative regularization results of the strongly convex case \citep{matet2017don} can be applied.

\section{Duality and Chambolle-Pock algorithm}
\label{sec:more_chambolle_pock}

\subsection{Duality}
\label{app:duality}

The Chambolle-Pock algorithm belongs to the class of primal-dual methods, designed to jointly solve \Cref{pb:cvxbias} (the primal problem), and its dual.
\Cref{pb:cvxbias} rewrites as
\begin{problem}
    \label{eq:prob}
    \argmin_{\w \in \R^p} \rr(\w)+\iota_{\{\Y\}}\left(\X\w\right) \enspace,
\end{problem}
with Lagrangian
\begin{equation}\label{eq:lagrangian}
    \cL(w, \theta) = J(w) + \langle \theta, \X w - \y \rangle \enspace.
\end{equation}
Fenchel-Rockafellar duality \cite[Sec. 3.6.2]{Pey15} can be applied to compute the dual problem; observing that $\iota_{\{\Y\}}^\star\left(\theta\right) = \langle \Y, \theta \rangle$, this dual reads:
\begin{problem}\label{pb:dual2}
	\argmin_{\theta \in \R^n} \left\{\rr^\star(-\X^\top \theta) + \iota_{\{\Y\}}^\star\left(\theta\right) \right\}
	= \argmin_{\theta \in \R^n} \left\{ \rr^\star(-\X^\top \theta) + \langle \Y, \theta \rangle \right\} \enspace.
\end{problem}
Assume that \Cref{eq:prob} admits a solution $\ws$ satisfying the following \textit{qualification condition},
\[
\tag{QC}
(\exists \thetas\in\R^n) \quad
-\X^\top\thetas \in \partial \rr (\ws) \enspace.\\
\]
Reasoning as in the proof of the Fenchel-Rockafellar duality theorem \cite[Thm. 3.51]{Pey15} it follows that strong duality holds, and $\thetas$ is a solution of \Cref{pb:dual2}. Primal-dual solutions are thus characterized by the first order conditions,
\begin{equation}\label{optcond1}
-\X^\top\thetas \in \partial \rr (\ws) \quad \text{and}\quad \X \ws = \Y \,.
\end{equation}
We stress the fact that we assume the existence of a solution $\ws$ satisfying (QC), but the primal solution is not necessarily unique.
On the other hand, from strong duality we get also that, for every primal solution, there exists a dual one such that (QC) (and so \Cref{optcond1,optcond2}) is verified.

\subsection{Chambolle-Pock algorithm}
\label{app:cp}
Consider the generic optimization problem
\begin{problem}\label{pb:rockafellar}\label{pb:prim_app}
	\min_\x \left\{ f(\x) + g(K\x) \right\},
\end{problem}
with Fenchel-Rockafellar dual problem given by
\begin{problem}\label{pb:dual_app}
	\min_\y \left\{ f^\star(-K^\top y) + g^\star(y) \right\}.
\end{problem}
In this general case, the Chambolle-Pock's algorithm (with interpolation parameter equal to 1) is given by
\begin{equation*}
\begin{split}
y_{k + 1} & = \prox_{\tau g^\star}(y_k + \tau K (2 x_{k} - x_{k - 1})),\\
\x_{k + 1} & = \prox_{\sigma f}(x_{k} - \sigma K^\top y_{k + 1}).
\end{split}
\end{equation*}
Notice that the CP algorithm, except for the interpolation, treats the primal and the dual problem in a symmetric way.
In particular, we can cast the method both for \Cref{pb:prim_app,pb:dual_app}.
In order to apply the latter to our dual problem, we set $f=\langle \Y, \cdot \rangle$, $g = \rr^\star$ and $K=-\X^\top$. Then $g^\star=\rr$, $\prox_{\sigma f}(\theta) = \theta - \sigma \Y$ and we recover \Cref{eq:cp}:
\begin{align*}
&\w_{k + 1} = \prox_{ \tau J } \left(\w_k - \tau \X^\top \left(2 \theta_k - \theta_{k - 1}\right) \right),\\
&\theta_{k + 1} = \theta_k + \sigma \left(\X\w_{k + 1} - \Y^\delta\right).
\end{align*}
The latter uses, in the update of the variable $\w$, an interpolation of $\theta$ with the value at the previous step.\\
As we already remarked, we could also apply the CP algorithm directly to the primal problem, setting $f=\rr$, $g=\iota_{\{\Y\}}$ and $K=\X$. Then $g^\star=\langle \Y, \cdot \rangle$ and $\prox_{\tau g^\star}(\theta) = \theta - \tau \Y$, leading to the following method:
\begin{align*}
&\theta_{k + 1} = \theta_k + \tau \left(\X\left(2\w_{k}-\w_{k-1}\right) - \Y^\delta\right),\\
&\w_{k + 1} = \prox_{ \sigma J } \left(\w_k - \sigma \X^\top\theta_{k + 1} \right).
\end{align*}
In this case, in the update of the variable $\theta$, we use an interpolation of $\w$.
In general, the two versions should not differ in a significant manner.
Nevertheless, the error we consider affects the data $\Y^{\delta}$ and so its nature is not symmetric.
Then, a different choice for the interpolation can play a role.
In this work, we put emphasis in Algorithm \eqref{eq:cp} because it is the one for which we have proximal errors in the non-extrapolated step.

\section{Proofs}

\subsection{Lemmas}
\label{sub:app_lemmas}
\optimalitylemma*
\begin{proof} \ \\
	{\em Step 1: the duality gap is the Bregman divergence.}
	Indeed, using $-\X^\top \thetas \in \partial \rr(\ws)$ and $\X \ws = \Y$:
	\begin{align}\label{breggap}
	\cL(\w', \thetas) - \cL(\ws, \theta')
	&= \rr(\w') - \rr(\ws) + \langle \thetas, \X \w' - \Y \rangle - \langle \theta', \X  \ws - \Y \rangle  \nonumber \\
	&= \rr(\w') - \rr(\ws) + \langle \X^\top \thetas, \w' - \ws \rangle = D_\rr^{-\X^\top \thetas}(\w', \ws)\enspace,
	\end{align}
	{\em Step 2: Zero duality gap plus feasibility implies primal optimality}
	We show that  if $\bar{v}\in\partial \rr(\ws)$ and $D_{\rr}^{\bar{v}}(\w',\ws)=0$, then $\bar{v}\in\partial \rr(\w')$.
	Indeed, $\rr(\w')-\rr(\ws)-\langle \bar{v}, \w'-\ws\rangle = 0$ and so, for all $z\in\R^p$,
	\begin{equation}\label{e2}
	J(z) \geq J(\ws) + \langle \bar{v}, z-\ws \rangle = J(\w') - \langle \bar{v}, \w'-\ws \rangle+ \langle \bar{v}, z-\ws \rangle = J(\w') + \langle \bar{v}, z-\w' \rangle.
	\end{equation}
	The statement follows by applying step 2 with  $\bar{v}=-\X^\top \thetas$.
\end{proof}

Next, we recall the result that allows us to control the non-vanishing error.
It is a discrete version of Bihari's Lemma and a particular case of Lemma 1 in \cite{schmidt2011convergence}, where the proof can be found.\\

\begin{lemma}\label{lem:u_n_upper_bound}
	Assume that $(u_j)$ is a non-negative sequence and that $\lambda\geq 0,  S\geq 0$ with $S \geq u_0^2$.
	If $u_t^2\leq S + \lambda \sum_{j=1}^{t}  u_j$, then
	\begin{equation*}
	\begin{split}
	u_t & \leq \frac{\lambda t}{2} + \left[S + \left( \frac{\lambda t}{2}\right)^2\right]^{\frac{1}{2}}.
	\end{split}
	\end{equation*}
	So, in particular,
	\begin{equation*}
	\begin{split}
	u_t & \leq \lambda t+ \sqrt{S}.
	\end{split}
	\end{equation*}
\end{lemma}

\subsection{Preliminary estimates}

\begin{lemma}[One step estimate]\label{onestep}
	Defining $\tilde \theta_k \eqdef 2 \theta_k - \theta_{k - 1}$, the updates of \eqref{eq:cp} for the noisy problem read as:
	\begin{align}\label{ppp1}
	&\w_{k + 1} = \prox_{ \tau J } \left(\w_k - \tau \X^\top \tilde \theta_k \right) \enspace,\\
	&\theta_{k + 1} = \theta_k + \sigma \left(\X\w_{k + 1} - \Y^\delta\right) \enspace.\label{ppp2}
	\end{align}
	Then, for any $(\w, \theta)\in\R^p\times\R^n$, we have the following estimate:
	\begin{equation}\label{mainonestep}
	\begin{split}
	&
	V(z_{k + 1} - z) - V(z_k - z) +  V(z_{k + 1} - z_k) +  \left[\cL(\w_{k + 1}, \theta) - \cL(\w, \theta_{k + 1})\right]+ \langle \theta_{k + 1} - \theta, \Y^\delta - \Y \rangle \\
	& + \langle \theta_{k + 1} - \tilde \theta_k, \X\left(\w - \w_{k + 1}\right)\rangle \leq 0.
	\end{split}
	\end{equation}
\end{lemma}
\begin{proof}
	Consider first \Cref{ppp1} and the firm non-expasiveness of the proximal-point. Then we get that, for any $w\in\R^p$,
	\begin{equation*}
	\begin{split}
	0 \geq \ &  \nor{\w_{k + 1} - \w}^2 - \nor{\left( \w_k - \tau \X^\top\tilde \theta_k \right) - \w}^2 + \nor{\w_{k + 1} - \left(\w_k - \tau \X^\top\tilde \theta_k\right)}^2 + 2\tau \left[J(\w_{k + 1})-J(\w)\right]\\
	= \ & \nor{\w_{k + 1} - \w}^2 - \nor{\w_k - \w}^2 + \nor{\w_{k + 1} - \w_k}^2 + 2\tau \left[J(\w_{k + 1}) - J(\w)\right] \\
	& + 2\tau\langle \X^\top\tilde \theta_k, \w_k - \w\rangle + 2\tau\langle \X^\top\tilde \theta_k, \w_{k + 1} - \w_k\rangle\\
	= \ & \nor{\w_{k + 1} - \w}^2 - \nor{\w_k - \w}^2 + \nor{\w_{k + 1} - \w_k}^2 + 2\tau \left[J(\w_{k + 1}) - J(\w)\right] + 2\tau\langle \tilde \theta_k, \X\left(\w_{k + 1} - \w\right)\rangle.
	\end{split}
	\end{equation*}
	\ \\
	Now consider \Cref{ppp1} and notice that the dual update can be re-written as $\theta_{k + 1} = \prox_{\sigma \langle \Y^\delta, \cdot \rangle }\left(\theta_k + \sigma \X\w_{k + 1} \right)$. Similarly as before, for any $\theta\in\R^n$,
	\begin{equation*}
	\begin{split}
	0 \geq \ & \nor{\theta_{k + 1} - \theta}^2 - \nor{\left(\theta_k + \sigma \X\w_{k + 1}\right) - \theta}^2 + \nor{\theta_{k + 1} - \left(\theta_k + \sigma \X\w_{k + 1}\right)}^2 + 2\sigma \left[\langle \Y^\delta, \theta_{k + 1}\rangle - \langle \Y^\delta, \theta\rangle\right]\\
	= \ & \nor{\theta_{k + 1} - \theta}^2 - \nor{\theta_k - \theta}^2 + \nor{\theta_{k + 1} - \theta_k}^2 + 2\sigma \langle \theta_{k + 1} - \theta, \Y^\delta\rangle \\
	& - 2\sigma\langle \theta_k - \theta, \X\w_{k + 1}\rangle-2\sigma\langle \theta_{k + 1} - \theta_k, \X\w_{k + 1}\rangle\\
	= \ & \nor{\theta_{k + 1} - \theta}^2 - \nor{\theta_k - \theta}^2 + \nor{\theta_{k + 1} - \theta_k}^2 + 2\sigma \langle \theta_{k + 1} - \theta, \Y^\delta-\X\w_{k + 1}\rangle.
	\end{split}
	\end{equation*}

	Recall that $z \eqdef (\w, \theta)$ and the definition of $V$ in \Cref{VVV}. Divide the first inequality by $2\tau$, the second one by $2\sigma$ and sum-up, to get
	\begin{equation*}
	\begin{split}
	0  \geq \ & V(z_{k + 1} - z) - V(z_k - z) +  V(z_{k + 1} - z_k) + \left[J(\w_{k + 1}) - J(\w)\right] \\
	&+ \langle \tilde \theta_k, \X\left(\w_{k + 1} - \w\right)\rangle + \langle \theta_{k + 1} - \theta, \Y^\delta - \X\w_{k + 1}\rangle.
	\end{split}
	\end{equation*}
	To conclude, compute
	\begin{equation*}
	\begin{split}
	& \left[J(\w_{k + 1}) - J(\w)\right] + \langle \tilde \theta_k, \X\left(\w_{k + 1} - \w\right)\rangle + \langle \theta_{k + 1} - \theta, \Y^\delta - \X\w_{k + 1}\rangle\\
	= \ &   \left[\cL(\w_{k + 1}, \theta) - \cL(\w, \theta_{k + 1})\right]-\langle \theta,  \X\w_{k + 1} - \Y \rangle  + \langle \theta_{k + 1}, \X\w - \Y \rangle \\
	& + \langle \tilde \theta_k, \X\left(\w_{k + 1} - \w\right)\rangle+  \langle  \theta_{k + 1} - \theta, \Y^\delta-\X\w_{k + 1}\rangle\\
	= \ & \left[\cL(\w_{k + 1}, \theta) - \cL(\w, \theta_{k + 1})\right]+ \langle \theta-\theta_{k + 1}, \Y \rangle - \langle \theta, \X\w_{k + 1}\rangle + \langle \theta_{k + 1}, \X \w \rangle\\
	& + \langle \tilde \theta_k, \X \w_{k + 1}\rangle- \langle \tilde \theta_k, \X \w\rangle+  \langle \theta_{k + 1} - \theta, \Y^\delta\rangle -\langle \theta_{k + 1} - \theta, \X\w_{k + 1}\rangle\\
	= \ & \left[\cL(\w_{k + 1}, \theta) - \cL(\w, \theta_{k + 1})\right]+ \langle \theta_{k + 1} - \theta, \Y^\delta - \Y \rangle \\
	& - \langle \theta, \X\w_{k + 1}\rangle+ \langle \theta_{k + 1}, \X \w\rangle + \langle \tilde \theta_k, \X\w_{k + 1}\rangle - \langle \tilde \theta_k, \X \w\rangle -  \langle \theta_{k + 1}, \X\w_{k + 1}\rangle+  \langle \theta, \X\w_{k + 1}\rangle\\
	= \ &  \left[\cL(\w_{k + 1}, \theta) - \cL(\w, \theta_{k + 1})\right]+ \langle \theta_{k + 1} - \theta, \Y^\delta - \Y \rangle + \langle \theta_{k + 1} - \tilde \theta_k, \X\left(\w -  \w_{k + 1}\right)\rangle.
	\end{split}
	\end{equation*}
\end{proof}

\begin{lemma}[First cumulating estimate]\label{cum1}
	Define $\omega:= 1 - \tau\sigma \opnor{\X}^2$. Then we have the following estimate:
	\begin{equation}\label{main}
	\begin{split}
	& \frac{\omega}{2\tau}\nor{\w_{k} - \ws}^2 + \frac{1}{2\sigma}\nor{\theta_{k} - \thetas}^2 - V(z_0 - \bar{z}) +  \sum_{t=1}^{k}\left[\mathcal{L}(\w_t,\thetas) - \mathcal{L}(\ws, \theta_t)\right]
	+\frac{\omega}{2\tau}\sum_{t=1}^{k}\nor{\w_t - \w_{t - 1}}^2  \\\leq \ \ &   \delta \sum_{t=1}^{k}\nor{\theta_t - \thetas}.
	\end{split}
	\end{equation}
\end{lemma}
\begin{proof}
	We start from \Cref{mainonestep}, switching the index from $k$ to $t$ and evaluating $(\w, \theta)$ at the saddle-point $(\ws,\thetas)$. Recall that $\tilde \theta_t \eqdef 2 \theta_t - \theta_{t - 1}$, to get
	\begin{equation*}
	\begin{split}
	& V(z_{t + 1} - \zs) - V(z_t - \zs) + V(z_{t + 1} - z_t)  +  \left[\cL(\w_{t + 1}, \thetas) - \cL(\ws, \theta_{t + 1})\right] \\
	\leq \ \ & - \langle \theta_{t + 1} - \left(2\theta_t - \theta_{t - 1}\right), \X\left(\ws - \w_{t + 1}\right)\rangle -\langle \theta_{t + 1} - \thetas, \Y^\delta - \Y \rangle\\
	\leq \ \ & - \langle \theta_{t + 1} - \theta_t, \X\left(\ws - \w_{t + 1}\right)\rangle +\langle \theta_t-\theta_{t - 1}, \X\left(\ws - \w_{t + 1}\right)\rangle + \delta \nor{\theta_{t + 1} - \thetas}\\
	= \ \ & - \langle \theta_{t + 1} - \theta_t, \X\left(\ws - \w_{t + 1}\right)\rangle +\langle \theta_t-\theta_{t - 1}, \X\left(\ws - \w_t\right)\rangle +\langle \theta_t-\theta_{t - 1}, \X\left(\w_t - \w_{t + 1}\right)\rangle\\
	& + \delta \nor{\theta_{t + 1} - \thetas}\\
	\leq \ \ & - \langle \theta_{t + 1} - \theta_t, \X\left(\ws - \w_{t + 1}\right)\rangle +\langle \theta_t-\theta_{t - 1}, \X\left(\ws - \w_t\right)\rangle \\
	& +\frac{1}{2\sigma}\nor{\theta_t-\theta_{t - 1}}^2 +\frac{\sigma}{2}\opnor{\X}^2\nor{\w_{t + 1} - \w_t}^2+ \delta \nor{\theta_{t + 1} - \thetas},
	\end{split}
	\end{equation*}
	where in the last estimate we used Cauchy-Schwartz and Young inequalities, the latter with parameter $\sigma$.
	Then, using the definition of $\omega \eqdef 1 - \tau\sigma\opnor{\X}^2$, we have
	\begin{equation*}
	\begin{split}
	& V(z_{t + 1} - \zs) - V(z_t - \zs) + \cL(\w_{t + 1}, \thetas) - \cL(\ws, \theta_{t + 1})\\
	& + \frac{\omega}{2\tau}\nor{\w_{t + 1} - \w_t}^2+\frac{1}{2\sigma}\nor{\theta_{t + 1} - \theta_t}^2 -\frac{1}{2\sigma}\nor{\theta_t-\theta_{t - 1}}^2\\
	\leq \ \ & - \langle \theta_{t + 1} - \theta_t, \X\left(\ws - \w_{t + 1}\right)\rangle +\langle \theta_t-\theta_{t - 1}, \X\left(\ws - \w_t\right)\rangle + \delta \nor{\theta_{t + 1} - \thetas}.
	\end{split}
	\end{equation*}
	Imposing $\theta_{-1}=\theta_0$, summing-up the latter from $t=0$ to $t=k-1$ and using the telescopic property, we get
	\begin{equation*}
	\begin{split}
	& V(z_{k} - \zs) - V(z_0 - \zs) + \sum_{t=0}^{k-1}\left[\cL(\w_{t + 1}, \thetas) - \cL(\ws, \theta_{t + 1})\right]\\
	&+\frac{\omega}{2\tau}\sum_{t=0}^{k-1}\nor{\w_{t + 1} - \w_t}^2 + \frac{1}{2\sigma} \nor{\theta_{k} - \theta_{k - 1}}^2 \\
	\leq \ \ & - \langle \theta_{k} - \theta_{k-1}, \X\left(\ws - \w_{k}\right)\rangle + \delta\sum_{t=0}^{k-1}\nor{\theta_{t + 1} - \thetas}\\
	\leq \ \ & \frac{1}{2\sigma} \nor{\theta_{k} - \theta_{k-1}}^2 +\frac{\sigma}{2} \opnor{\X}^2 \nor{\w_{k} - \ws}^2 + \delta \sum_{t=1}^{k}\nor{\theta_t - \thetas},
	\end{split}
	\end{equation*}
	where in the last inequality we used again Cauchy-Schwartz and Young inequalities with parameter $\sigma$. Reordering, we obtain the claim.
\end{proof}
\begin{lemma}[Second cumulative estimate]\label{cum2}
	For $\varepsilon>0$ and $\eta = \frac{1 + \varepsilon}{1 - \varepsilon}\geq 1$, define $\omega := \varepsilon - \sigma \tau\opnor{\X}^2$.
	Then we have
	\begin{equation}\label{mainn}
	\begin{split}
	& V(z_{k} - \zs) - V(z_0 - \zs) +\frac{\omega}{2\tau\varepsilon}\sum_{t=1}^{k}\nor{\w_{t} - \w_{t-1}}^2 +\frac{\sigma\varepsilon}{2\eta}\sum_{t=1}^{k}\nor{\X \w_{t}- \Y}^2\\
	& +\sum_{t=1}^{k}\left[\cL(\w_{t }, \theta) - \cL(\w, \theta_{t})\right] \ \ \leq \  \ \delta \sum_{t=1}^{k}\nor{\theta_t - \theta}+ \frac{\sigma \left(\eta-1\right)\delta^2k}{2}\enspace.
	\end{split}
	\end{equation}
\end{lemma}

\begin{proof}
	In a similar fashion as in the previous proof, we start again from \Cref{mainonestep}, switching the index from $k$ to $t$ and evaluating $(\w, \theta)$ at the saddle-point $(\ws,\thetas)$. Since $\tilde \theta_t = \theta_t + (\theta_t - \theta_{t - 1}) = \theta_t + \sigma (\X \w_t - \Y)$ and $\theta_{t+1} - \theta_t = \sigma (\X \w_{t+1} - \Y^{\delta})$, we get
	\begin{equation*}
	\begin{split}
	& V(z_{t + 1} - \zs) - V(z_t - \zs) + \frac{1}{2 \tau} \nor{\w_{t+1} - \w_t}^2 + \frac{\sigma}{2}\nor{\X \w_{t+1} - \Y^\delta}^2 +  \left[\cL(\w_{t + 1}, \thetas) - \cL(\ws, \theta_{t + 1})\right]\\
	\leq \ & \langle  \theta_{t+1} - \theta_t - \sigma \left(\X \w_{t} - \Y^\delta\right) , \X \w_{t+1} - \Y \rangle +\langle  \theta_{t+1} - \thetas, \Y - \Y^\delta\rangle\\
	= \ & \sigma\langle \X\left( \w_{t+1} - \w_t\right), \X \w_{t+1} - \Y\rangle + \langle  \theta_{t+1} - \thetas, \Y - \Y^\delta\rangle.
	\end{split}
	\end{equation*}
	Now compute
	\begin{equation*}
	\begin{split}
	\frac{\sigma}{2}\nor{\X \w_{t+1} - \Y^\delta}^2
	& = \frac{\sigma}{2}\nor{\X \w_{t+1} - \Y}^2 + \frac{\sigma}{2}\nor{\Y^\delta - \Y}^2  - \sigma \langle \X \w_{t+1} - \Y, \Y^\delta - \Y\rangle.
	\end{split}
	\end{equation*}
	So,
	\begin{equation*}
	\begin{split}
	& V(z_{t + 1} - \zs) - V(z_t - \zs) + \frac{1}{2\tau} \nor{\w_{t+1} - \w_t}^2 + \frac{\sigma}{2}\nor{\X \w_{t+1} - \Y}^2  + \left[\mathcal{L}(\w_{t+1},\thetas) - \mathcal{L}\left(\ws,\theta_{t+1}\right)\right]\\
	\leq \ & \sigma\langle \X\left( \w_{t+1} - \w_t\right), \X \w_{t+1} - \Y\rangle +\langle  \theta_{t+1} - \thetas, \Y - \Y^\delta\rangle   + \sigma \langle \X \w_{t+1} - \Y, \Y^\delta  - \Y \rangle - \frac{\sigma}{2}\nor{\Y^\delta - \Y}^2\\
	\leq \ & \frac{\sigma\opnor{\X}^2}{2\varepsilon}\nor{\w_{t+1} - \w_t}^2 +\frac{\varepsilon\sigma}{2}\nor{\X \w_{t+1}- \Y}^2+ \delta \nor{\theta_{t+1} - \thetas} - \frac{\sigma}{2}\nor{\Y^\delta - \Y}^2\\
	& +\frac{\sigma}{2\eta}\nor{\X \w_{t+1}- \Y}^2 + \frac{\sigma \eta}{2}\nor{\Y^\delta - \Y}^2.
	\end{split}
	\end{equation*}
	In the last inequality we used three times Cauchy-Schwartz inequality, the bound on the error given by $\nor{\Y^{\delta}-\Y}\leq \delta$ and two times Young inequality with parameters $\varepsilon>0$ and $\eta = \frac{1 + \varepsilon}{1 - \varepsilon} > 0$.
	Then, re-ordering and recalling the definitions of $\omega := \varepsilon - \sigma \tau\opnor{\X}^2$,  we obtain
	\begin{equation*}
	\begin{split}
	& V(z_{t + 1} - \zs) - V(z_t - \zs)  + \frac{\omega}{2\tau \varepsilon}\nor{\w_{t+1} - \w_t}^2 + \frac{\sigma\varepsilon}{2\eta}\nor{\X \w_{t+1}- \Y}^2 +  \left[\mathcal{L}(\w_{t+1},\thetas) - \mathcal{L}\left(\ws,\theta_{t+1}\right)\right]\\
	\leq \ &  \delta \nor{\theta_{t+1} - \thetas} + \frac{\sigma \left(\eta-1\right)\delta^2}{2}.
	\end{split}
	\end{equation*}
	Summing-up the latter from $t=0$ to $t=k-1$, by telescopic property, we get
	\begin{equation*}
	\begin{split}
	& V(z_{k} - \zs) - V(z_0 - \zs) +\frac{\omega}{2\tau\varepsilon}\sum_{t=0}^{k-1}\nor{\w_{t + 1} - \w_t}^2 +\frac{\sigma\varepsilon}{2\eta}\sum_{t=0}^{k-1}\nor{\X \w_{t+1}- \Y}^2\\
	& + \sum_{t=0}^{k-1}\left[\cL(\w_{t + 1}, \theta) - \cL(\w, \theta_{t + 1})\right] \ \ \leq \ \ \delta \sum_{t=0}^{k-1}\nor{\theta_{t+1}- \thetas}+ \frac{\sigma \left(\eta-1\right)\delta^2k}{2}\enspace.
	\end{split}
	\end{equation*}
	By trivial manipulations, we get the claim.
\end{proof}

\subsection{Proof of \Cref{early}}
\label{sec_early}
\earlystoppinggap*
\begin{proof}
Inequality in \Cref{main} holds true for every $k\geq 1$. Then, recalling that $\mathcal{L}(\w,\thetas) - \mathcal{L}(\ws, \theta)\geq 0 $ for every $\left(\w,\theta\right)\in\R^p\times\R^n$ and that $\omega\geq 0$ by assumption, for every $t\geq 1$ we have that
\begin{equation}\label{eq:bound_norm_y}
\nor{\theta_t - \thetas}^2  \ \leq \  2 \sigma V(z_0 - \bar{z}) + 2\sigma\delta \sum_{j=1}^{t}\nor{\theta_j - \thetas}.
\end{equation}
Apply \Cref{lem:u_n_upper_bound} to \Cref{eq:bound_norm_y} with $u_j=\nor{\theta_j - \thetas}$, $S=2 \sigma V(z_0 - \bar{z})$ and $\lambda=2\sigma\delta$, to get
\begin{align*}
\nor{\theta_t - \thetas}  \ & \leq \  2\sigma\delta t + \sqrt{2 \sigma V(z_0 - \bar{z}) }.
\end{align*}
In particular, for $1\leq t\leq k$, we have
\begin{align}\label{esttt}
\nor{\theta_t - \thetas}   \leq 2\sigma\delta k + \sqrt{2\sigma V(z_0 - \bar{z})}.
\end{align}
Insert the latter in \Cref{main}, to obtain
\begin{equation*}
\begin{split}
\sum_{t=1}^{k}\left[\mathcal{L}(\w_t,\thetas) - \mathcal{L}(\ws, \theta_t)\right] \ & \leq \ V(z_0 - \bar{z}) + \delta \sum_{t=1}^{k} \left(2\sigma\delta k + \sqrt{2\sigma V(z_0 - \bar{z})}\right)\\
& = \ V(z_0 - \bar{z})  + \delta k \sqrt{2\sigma V(z_0 - \bar{z})} + 2 \sigma\delta^2k^2\\
& \leq \ \left(\sqrt{V(z_0 - \bar{z})}  +  \sqrt{2\sigma} \delta k\right)^2 \enspace.
\end{split}
\end{equation*}
By Jensen's inequality, we get the claim.

For the second result, recall that, from \Cref{esttt}, we have
\begin{equation*}
\delta \sum_{t=1}^{k} \nor{\theta_{t} - \thetas} \leq \sqrt{2\sigma V(z_0 - \zs)}\delta k + 2\sigma \delta^2 k^2.
\end{equation*}
Inserting the latter in \Cref{mainn}, we get
\begin{equation*}
\begin{split}
\frac{\sigma\varepsilon}{2\eta}\sum_{t=1}^{k}\nor{\X \w_{t} - \Y}^2 \ & \leq \ \delta \sum_{t=1}^{k} \nor{\theta_{t} - \thetas} + \frac{\sigma \left(\eta-1\right)\delta^2 k}{2}+ V(z_0 - \zs) \enspace\\
& \leq \sqrt{2\sigma V(z_0 - \zs)}\delta k + 2\sigma \delta^2 k^2+ \frac{\sigma \left(\eta-1\right)\delta^2 k}{2}+ V(z_0 - \zs).
\end{split}
\end{equation*}
By Jensen's inequality, rearranging the terms, and taking $\eta = \frac{1 + \varepsilon}{1 - \varepsilon}$, we get the claim.
\end{proof}